\documentclass{article} 
\usepackage{iclr2027_conference,times}

\usepackage{amsmath,amssymb,amsfonts}
\usepackage{amsthm}
\usepackage{bm}

\usepackage{graphicx}
\usepackage{booktabs}
\usepackage{multirow}
\usepackage{array}
\usepackage{makecell}
\usepackage{xcolor}
\usepackage{tikz}
\usetikzlibrary{arrows.meta,positioning,calc}

\usepackage{algorithm}
\usepackage{algpseudocode}

\usepackage{placeins}
\usepackage{microtype}
\usepackage{hyperref}
\hypersetup{
  pdftitle={GraphVQ: Structure-Aware Autoregressive Decoding over Context-Quantized Graph Tokens},
  pdfauthor={Yuxiang Yao, Zijun Zhao},
  pdfsubject={Machine Learning (cs.LG)},
  pdfkeywords={graph generation, vector quantization, autoregressive models, graph tokenization, graph foundation models}
}
\usepackage{url}

\newcommand{\method}{\textsc{GraphVQ}}

\newcommand{\R}{\mathbb{R}}
\newcommand{\od}{\odot}
\DeclareMathOperator*{\argmin}{arg\,min}
\newtheorem{proposition}{Proposition}

\title{GraphVQ: Structure-Aware Autoregressive Decoding over Context-Quantized Graph Tokens}

\author{Yuxiang Yao$^{\dagger}$ \\
Research Center, \\ 
China Life Insurance Company Ltd.\\
\texttt{yaoyuxiangyyx2023@e-chinalife.com}
\And
Zijun Zhao$^{\dagger*}$ \\
School of Computer Science and Technology\\
Beijing Institute of Technology\\
\texttt{zhaozijun@bit.edu.cn}
}

\iclrfinalcopy 

\begin{document}

\maketitle
\lhead{Preprint} 
\begingroup
\renewcommand\thefootnote{}\footnotetext{$^{\dagger}$Equal contribution. $^{*}$Corresponding author.}%
\addtocounter{footnote}{-1}%
\endgroup

\begin{abstract}
Graph foundation models need a discrete token representation, but casting a graph as a generatable token sequence faces a structural obstacle: edges spanning beyond the serialization window cannot be emitted in one pass---so one-pass autoregressive generators systematically under-produce cycles---and a single global condition cannot tell candidate edges apart. \method{} removes both obstacles: node contexts---features plus a local edge mask under multi-order breadth-first serialization---are quantized into a shared codebook by a VQ-VAE with BCE-calibrated Bernoulli edge decoding, and a second-stage \emph{structure-aware} decoder emits the global adjacency conditioned on token-derived pair features, whose necessity over any global-summary condition is formalized in a scoped impossibility result. The tokenizer reconstructs node features at $0.86$--$0.99$ accuracy and decodes local edges at AUROC $\geq\!0.89$ (ECE $\leq\!0.007$). Under one same-split protocol on four datasets, pair conditioning improves orbit MMD $0.248\!\to\!0.174$ on PROTEINS and $3.4\times$ on a ring stress test, and vanishes on a random-label control---the signature of attribute--topology coupling---so the gain is claimed exactly where attributes carry edge-relevant signal. \method{} ranks first among learned generators on PROTEINS, ties for first on SYN-COMM, and improves orbit MMD $2.7$--$17\times$ over one-stage generation on three datasets, with seed-level bootstrap intervals confirming the rankings are not seed noise; on MUTAG the unweighted edge target under-generates and is reported as such. These results locate the structural control of autoregressive graph generation in the granularity of the condition: pair-level token context turns a quantized vocabulary into a usable capacity axis for distribution-faithful graph generation and future token-level pretraining.
\end{abstract}

\section{Introduction}
\label{sec:intro}

Graphs describe molecules, proteins, social and citation networks, and knowledge bases, and graph neural networks are the standard tool for supervised learning over them~\cite{xu2019gin,kipf2017gcn,hamilton2017sage}. Yet the field has not produced the analogue of a language or vision foundation model: a single pretrained model that transfers across graphs of different sizes, feature spaces, and domains~\cite{wang2024gft,xia2024opengraph}. A core obstacle is representational. Text and images are mapped to discrete token sequences before a shared autoregressive or masked model is trained at scale~\cite{kaplan2020scaling,hoffmann2022chinchilla}; graphs resist this recipe because nodes are not naturally ordered, node features vary in dimensionality, and topologies differ across datasets. Building a discrete, transferable graph vocabulary---a graph tokenizer---is therefore a central problem for graph foundation models~\cite{guo2026graphtoken}.

Vector quantization (VQ) is the natural candidate: encode local graph contexts and quantize them into a finite shared codebook, following VQ-VAE~\cite{oord2017vqvae} and the graph-specific designs GQT~\cite{wang2025gqt} and GFT~\cite{wang2024gft}. Discrete tokens bring a finite alphabet, cross-entropy training, and a uniform token-level interface---the ingredients of token-level pretraining. We do not claim discretization is the only route to autoregressive generation; we claim something testable: the discrete vocabulary earns its keep inside the generative pipeline, verified against controls that vary the token representation (Sec.~\ref{sec:exp}).

The second ingredient is structural. Under any fixed serialization, an edge whose endpoints lie more than $W$ positions apart cannot be emitted in one pass, and a cycle needs every one of its edges---so the residual long-span edges that BFS leaves uncovered break cycles and motifs, and one-pass generators systematically under-produce cycles (Fig.~\ref{fig:windowcov}, Appendix~\ref{app:figures}). A two-stage scheme---sample node tokens first, then decode the global adjacency---removes the window limit, but only if the edge phase is conditioned on information that actually predicts edges. A single global summary such as the mean node feature is provably insufficient \emph{when it carries no endpoint information} (Proposition~\ref{prop:global}, scoped accordingly). We verify this directly: enriching the condition from the mean feature to token-derived pair features improves orbit MMD from $0.248$ to $0.174$ on PROTEINS (paired $t{=}{-}6.1$) and by $3.4\times$ on a synthetic ring stress test, while on a random-label control---which removes attribute--topology coupling but keeps the local adjacency context---the gain vanishes, exactly as the attribute-coupling hypothesis predicts.

We present \method{}, a minimal pipeline that combines both ideas. Nodes are serialized in breadth-first (BFS) order with random roots and tie-breaking; each node context is quantized into a shared codebook; a two-layer GRU prior factorizes the token sequence; and generation is two-stage, with Stage~B an autoregressive edge decoder conditioned on token-derived pair features over the blocked upper-triangular adjacency. Our contributions are:
\begin{itemize}
  \item \textbf{Structure-aware autoregressive decoding over token-derived pair contexts.} Stage~B builds per-pair features from token embeddings, their interaction, position, and a graph-level summary, pools them per chunk, and emits chunked bits as masked independent Bernoullis with positional readouts. Replacing the global mean condition with pair features improves orbit MMD $0.248\!\to\!0.174$ on PROTEINS ($t{=}{-}6.1$) and $3.4\times$ on a ring stress test, while the gain disappears on a random-label control---the signature of attribute--topology coupling.
  \item \textbf{A calibrated context-quantized graph tokenizer.} BCE-calibrated Bernoulli edge decoding reaches AUROC $\geq\!0.95$ and ECE $\leq\!0.005$ on real data, and removing the local-structure context collapses generation (orbit MMD $0.59$--$1.17$): the quantized context, not the attributes alone, carries the structural signal.
  \item \textbf{Competitive end-to-end generation under a same-split protocol.} Against seven baselines and two controls on four datasets (five seeds, $10{,}000$ samples, pre-registered mean-rank over ten methods), \method{} ranks first among learned generators on PROTEINS, ties Raw for first on SYN-COMM, and beats the continuous-latent control on three of four datasets.
  \item \textbf{Separate conditional-reconstruction and free-generation evaluation, and order-robust serialization.} Conditional Stage~B metrics under true conditions are reported apart from free-generation distribution metrics, the condition gap ($0.005$--$1.68$ nats/bit) is reported without over-reading, and shuffled-condition plus unconditional controls locate where the condition earns its gain. Multi-order BFS augmentation lowers held-out sequence NLL from $1.44$ to $1.10$ on MUTAG.
\end{itemize}

\section{Related Work}
\label{sec:related}

\paragraph{Graph tokenization and graph foundation models.} Recent work shares the premise that graphs must be expressed as tokens before foundation-model training: GFT~\cite{wang2024gft} builds a transferable tree vocabulary for pretraining and prompting; OpenGraph~\cite{xia2024opengraph} learns a continuous, topology-aware projection for zero-shot transfer; Tokenphormer~\cite{zhou2025tokenphormer} encodes multi-token structure for node classification; GQT~\cite{wang2025gqt} tokenizes PPR-based node contexts with residual VQ and drives molecular generation with a GPT-style prior. Closest in spirit, Guo and Diao~\cite{guo2026graphtoken} connect graphs to transformers through \emph{reversible} serializations and BPE-style token vocabularies. Prior work targets transfer, molecule generation with heavy transformer priors, or tokenization as a forward interface; none addresses how the \emph{decoding} side should exploit the tokens. \method{} contributes exactly this: a token-conditioned autoregressive decoder with pair-level features, plus controls that isolate the value of discretization itself, so the token argument rests on measurements rather than assertion.

\paragraph{Autoregressive graph generation.} GraphRNN~\cite{you2018graphrnn} pioneered autoregressive graph generation with a BFS ordering and an edge-level RNN; GRAN~\cite{liao2019gran} generates the adjacency in blocks with attention; GraphARM~\cite{kong2023grapharm} interleaves autoregression with absorbing-state diffusion. One-shot models decode whole graphs from a single latent~\cite{simonovsky2018graphvae,decao2018molgan,martinkus2022spectre}; the diffusion family spans discrete~\cite{vignac2023digress}, continuous score-based~\cite{jo2022gdss}, and mixture~\cite{jo2024grum} variants. A shared weakness of row- or edge-level autoregression is that cycles need long-range back-edges, so one-pass schemes systematically under-generate them; \method{} addresses this structurally with a second stage that decodes all candidate pairs conditioned on the generated token sequence.

\paragraph{Vector quantization and vocabulary capacity.} VQ-VAE~\cite{oord2017vqvae} introduced discrete latent codes with commitment loss and EMA codebook updates; VQ-VAE-2~\cite{razavi2019vqvae2} scaled it hierarchically, and VQGAN~\cite{esser2021vqgan} added adversarial decoding. Known failure modes include codebook collapse and low utilization, usually monitored rather than modeled. \method{} inherits this machinery and adds a graph-specific lesson: BCE plus Bernoulli sampling aligns edge density with the data prior by construction, and utilization statistics (active codes, perplexity, frequency Gini) act as first-class tokenizer diagnostics. Scaling laws for language~\cite{kaplan2020scaling,hoffmann2022chinchilla} and GNNs~\cite{liu2024scaling} leave vocabulary size under-studied; our diagnostics show the vocabulary is used as far as the data demands---a capacity-truncation signature making utilization, not nominal size, the effective bound.

\section{Methodology}
\label{sec:method}

\subsection{Overview and Notation}
\label{sec:overview}

Fig.~\ref{fig:framework} shows \method{}. Given an attributed graph $G=(V,E,X)$ with $N{=}|V|$ nodes and feature matrix $X\!\in\!\R^{N\times d}$, Stage~1 serializes the graph in BFS order and tokenizes each node context into a discrete code from a shared codebook (a VQ-VAE). Stage~2 trains a GRU prior over the token sequence. Stage~3 generates in two phases: node tokens first (Stage~A), then the global adjacency through a per-pair conditioned autoregressive edge decoder (Stage~B). Table~\ref{tab:notation} in Appendix~\ref{app:notation} summarizes the notation.

\begin{figure}[!htb]
\centering
\includegraphics[width=0.78\linewidth]{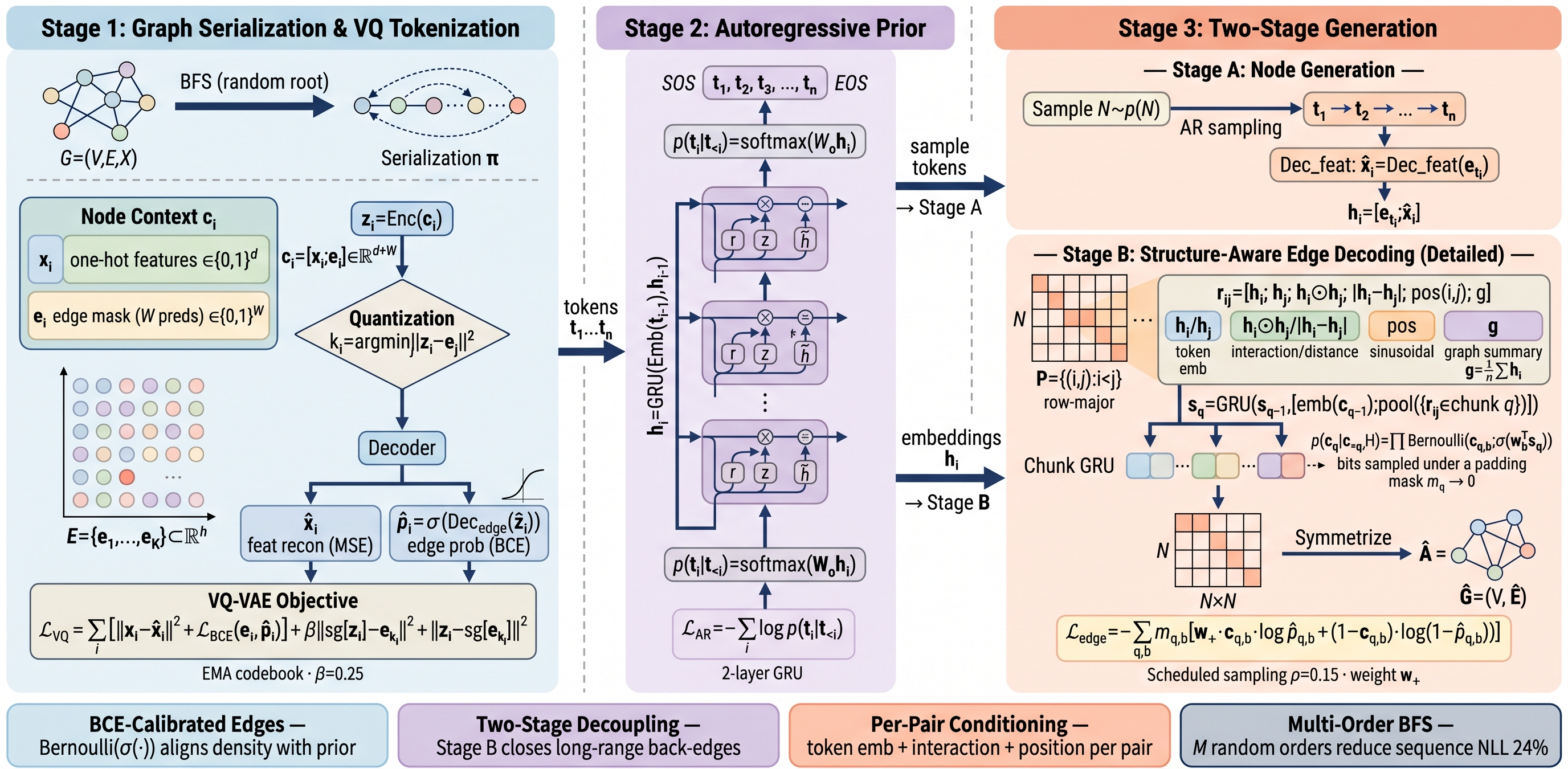}
\caption{The three stages of \method{}. \textbf{Stage 1} (training): the graph is serialized in BFS order (random roots/tie-breaking, multi-order augmentation); each node context $c_i$---one-hot features plus an edge mask over the preceding $W$ nodes---is encoded, quantized to the nearest codebook entry (the discrete token), and decoded into features and BCE-calibrated edge probabilities. \textbf{Stage 2} (training): a GRU prior factorizes the token sequence. \textbf{Stage 3} (inference): Stage~A samples node tokens and decodes features; Stage~B builds per-pair features $r_{ij}$, chunks the upper triangle into $B$-bit groups, and autoregressively emits each chunk from a GRU whose input pools the chunk's pair features; padding bits are forced to zero and the symmetrized matrix is the generated graph. Solid arrows: forward flow; dashed: training objectives.}
\label{fig:framework}
\end{figure}

\subsection{Graph Serialization and VQ Tokenization}
\label{sec:tokenizer}

We serialize a graph into a sequence of \emph{node contexts}. Nodes are ordered by BFS, which places topologically close nodes near each other and keeps most edges local---the observation that motivated GraphRNN~\cite{you2018graphrnn}. Because BFS depends on the root and same-level tie-breaking, the serialization is not permutation-invariant; we treat this explicitly by drawing a fresh random root and tie-breaking during training, with multi-order augmentation ($M$ orders per graph), and aggregating held-out likelihoods over eight orders at evaluation (Sec.~\ref{sec:exp}). The model is therefore permutation-aware \emph{through augmentation}, not invariant.

The context of the $i$-th node in the order concatenates its one-hot features $x_i\!\in\{0,1\}^{d}$ with a binary edge mask $e_i\!\in\{0,1\}^{W}$ over the $W$ preceding nodes:
\begin{equation}
c_i = [x_i;\,e_i] \in \R^{d+W},
\qquad e_{i,w}=\mathbb{1}\big[(v_{\pi(i)},v_{\pi(i-w)})\in E\big],
\label{eq:context}
\end{equation}
where $\pi$ is the order and entries beyond the window are zero-padded. $W$ trades locality against reach: an edge spanning more than $W$ positions cannot be emitted in a single pass---expressibility is set by the \emph{edge span} induced by the order. BFS keeps most edges local while leaving a residual of long-span edges that break cycles and motifs: the structural gap that Stage~B (Sec.~\ref{sec:generation}) closes. Removing the edge mask entirely (``attribute-only'' context, Sec.~\ref{sec:exp}) collapses generation (orbit MMD $0.59$--$1.17$, connectivity $\approx\!0$): the local-structure context, not the attributes alone, carries the tokenizer's structural information.

The context is encoded and quantized against a shared codebook $E=\{e_1,\dots,e_K\}\subset\R^{h}$:
\begin{equation}
z_i=\mathrm{Enc}(c_i),\qquad
k_i=\argmin_{j}\big\lVert z_i-e_j\big\rVert^2,\qquad
\hat z_i=e_{k_i}.
\label{eq:quantize}
\end{equation}
The decoder reconstructs the features and models edges \emph{explicitly} as calibrated probabilities $\hat p_i\!\in\!(0,1)^{W}$:
\begin{equation}
[\hat x_i;\,\hat p_i]=\mathrm{Dec}(\hat z_i),
\qquad \hat p_i=\sigma\big(\mathrm{Dec}_{\mathrm{edge}}(\hat z_i)\big).
\label{eq:decode}
\end{equation}
The tokenizer is trained as a VQ-VAE with feature reconstruction (MSE), explicit edge modeling (BCE), and the standard codebook-update and commitment terms ($\mathrm{sg}$ denotes stop-gradient):
\begin{equation}
\begin{aligned}
\mathcal{L}_{\mathrm{VQ}}
= \sum_{i}\Big[&\big\lVert x_i-\hat x_i\big\rVert^2
+\mathcal{L}_{\mathrm{BCE}}(e_i,\hat p_i)\Big]\\
&+\beta\,\big\lVert\mathrm{sg}[z_i]-e_{k_i}\big\rVert^2
+\big\lVert z_i-\mathrm{sg}[e_{k_i}]\big\rVert^2.
\end{aligned}
\label{eq:vqloss}
\end{equation}
Codebook entries are updated by EMA, and utilization (active codes, perplexity, frequency Gini) is monitored to detect collapse. Modeling edges with \emph{unweighted} BCE is decisive for calibration: the trained probability already encodes the empirical edge density, so drawing edges from $\mathrm{Bernoulli}(\hat p_i)$ aligns the generated density with the data prior (Table~\ref{tab:tokenizer}, Appendix~\ref{app:stats}: Brier $\leq\!0.065$, ECE $\leq\!0.005$). The same claim does \emph{not} transfer to the weighted Stage~B loss, whose optimum is deliberately biased (Sec.~\ref{sec:generation}).

\subsection{Autoregressive Prior}
\label{sec:prior}

The tokenizer maps each graph to a discrete sequence $t_1,\dots,t_n$ of code indices. We train a GRU prior that factorizes this sequence autoregressively, wrapped by SOS and EOS:
\begin{equation}
\begin{aligned}
p(t_1,\dots,t_n)&=\textstyle\prod_{i=1}^{n} p(t_i\mid t_{<i}),\\
p(t_i\mid t_{<i})&=\mathrm{softmax}(W_o h_i),
\end{aligned}
\label{eq:prior}
\end{equation}
\begin{equation}
\begin{aligned}
h_i&=\mathrm{GRU}\big(\mathrm{Emb}(t_{i-1}),\,h_{i-1}\big),\\
\mathcal{L}_{\mathrm{AR}}&=-\textstyle\sum_{i}\log p(t_i\mid t_{<i}).
\end{aligned}
\label{eq:ar}
\end{equation}
A two-layer GRU suffices and keeps the prior deliberately minimal; the architecture is a drop-in slot that larger transformer priors can occupy without changing the tokenizer or the generation scheme.

\subsection{Two-Stage Generation with Per-Pair Conditioning}
\label{sec:generation}

Naive generation samples one token sequence and decodes it, but the window $W$ caps each node's reach: longer-span edges are structurally unreachable in one pass, and every missing edge breaks the cycles and motifs that contain it. We therefore generate in two phases, with the second phase conditioned on token-derived pair information from the full token sequence.

\emph{Stage A (nodes).} Sample the graph size $N$ from the training distribution, then autoregressively sample tokens and decode them into one-hot features:
\begin{equation}
t_1,\dots,t_N \sim p(t_i\mid t_{<i}),
\qquad
\hat x_i=\mathrm{Dec}_{\mathrm{feat}}(e_{t_i}),
\label{eq:stagea}
\end{equation}
Each node carries a token embedding $h_i=e_{t_i}$, and the pair representations below concatenate it with the decoded feature $\hat x_i$, so Stage~B conditions on both the quantized representation and its deterministic decoding.

\emph{Stage B (edges).} Consider the candidate pair set $\mathcal{P}=\{(i,j):1\!\leq\!i\!<\!j\!\leq\!N\}$ in row-major order. For each pair we build a feature vector
\begin{equation}
r_{ij}=\big[\,h_i;\,h_j;\,h_i\od h_j;\,|h_i-h_j|;\,\mathrm{pos}(i,j);\,g\,\big],
\label{eq:pairfeat}
\end{equation}
where $\od$ is element-wise product, $\mathrm{pos}(i,j)$ are normalized sinusoidal positional features of the pair, $g=\frac1N\sum_i h_i$ is a graph-level summary, and $h_i$ here denotes the concatenation $[e_{t_i};\,\hat x_i]$. The $N(N{-}1)/2$ pairs are grouped into chunks of $B$ consecutive bits, with the final chunk zero-padded and its padding positions recorded in a mask $m_q\!\in\!\{0,1\}^{B}$ (padding bits carry zero features and mask value $0$). A chunk-level GRU aggregates each chunk's pair features into its input (the $B$ bits of a chunk therefore share one pooled summary and are distinguished by their positional readout heads $w_b$; whether this chunk-pooled form retains per-pair position information is tested in Sec.~\ref{sec:exp}):
\begin{equation}
s_q=\mathrm{GRU}\Big(s_{q-1},\ \big[\mathrm{emb}(c_{q-1});\ \mathrm{pool}(\{r_{ij}:(i,j)\!\in\!q\})\big]\Big),
\label{eq:stageb-gru}
\end{equation}
\begin{equation}
p(c_q\mid c_{<q},\mathcal{H})=\prod_{b=1}^{B}\mathrm{Bernoulli}\big(c_{q,b};\ \sigma(w_b^{\top} s_q)\big),
\label{eq:stageb}
\end{equation}
where $\mathrm{pool}$ is the masked mean of the projected pair features in chunk $q$, $\mathrm{emb}(c_{q-1})$ embeds the previous chunk's bits (a learned start vector at $q{=}1$), and $\mathcal{H}=\{h_i\}_{i=1}^{N}$ is the full token-derived conditioning. The optional positive-class weight $w_{+}$ \emph{re-targets} density rather than calibrating probability: the weighted-BCE optimum is $q^{*}=wp/(wp{+}1{-}p)$, so with $w_{+}{=}3$ a true edge probability of $0.1$ is fit by $0.25$---a deliberate bias toward the positive class (ablated in Table~\ref{tab:cond}, Appendix~\ref{app:ablations}; Fig.~\ref{fig:probcheck}, Appendix~\ref{app:figures}, verifies this on synthetic Bernoulli data, including the sampling rate). Weighted outputs are therefore \emph{not} calibrated probabilities; they are reported as density-matched samples, and the main configuration uses the unweighted target with validation-set temperature calibration (Sec.~\ref{sec:exp}).
Bits are sampled independently within a chunk; a within-chunk autoregressive head is ablated in Sec.~\ref{sec:exp}. The training loss is a masked binary cross-entropy over all chunks, with the optional positive-class weighting $w_{+}$ defined above:
\begin{equation}
\begin{aligned}
\mathcal{L}_{\mathrm{edge}}=-\sum_{q,b} m_{q,b}\big[&w_{+}\,c_{q,b}\log\hat p_{q,b}\\
&+(1{-}c_{q,b})\log(1{-}\hat p_{q,b})\big].
\end{aligned}
\label{eq:edgeloss}
\end{equation}

Two further design points make the model robust to train--sample mismatch. (i)~\emph{Scheduled sampling:} during training, each node token embedding is replaced by a random codebook entry with probability $\rho$ (we use $\rho{=}0.15$), so the edge model observes the kind of errors Stage~A will make. (ii)~\emph{Condition diagnostics:} we evaluate the edge NLL of held-out graphs under three conditions---\emph{oracle} (tokens of the true features), \emph{reconstructed} (deterministic Stage-A reconstruction), and \emph{generated} (one prior sample). The oracle-to-generated gap is $0.005$/$0.009$ nats/bit on MUTAG/PROTEINS and larger on the synthetic benchmarks ($0.10$/$1.68$ on SYN-COMM/SYN-RING-ROLE), localizing the residual to Stage~A (Sec.~\ref{sec:exp}); it is not read as certification---a condition-ignoring decoder would show the same near-zero gap---so condition dependence is tested directly with shuffled-condition and unconditional controls (Sec.~\ref{sec:exp}). Algorithm~\ref{alg:gen} states the procedure; candidate pairs are constructed from the index set $\mathcal{P}$, and at no point does inference read from an adjacency matrix that does not yet exist.

The pair-feature design is not a luxury: a purely global condition provably cannot do the job \emph{within its scope}, as follows.

\begin{proposition}[Insufficiency of a global-summary condition]
\label{prop:global}
Let $s(X)$ be any permutation-invariant summary of the node attributes \emph{alone} (e.g., the mean feature $\bar c$), and let $q(A_{ij}\!\mid\! s(X))$ be an edge model whose condition is $s(X)$ only---no endpoint-specific, positional, or structural inputs. Suppose two graphs $G_1,G_2$ on $N$ nodes share the same attribute multiset, $\{x_i^{(1)}\}\!=\!\{x_i^{(2)}\}$, but differ in attribute--topology coupling, $p_{G_1}(A_{ij}{=}1\!\mid\! x_i,x_j)\neq p_{G_2}(A_{ij}{=}1\!\mid\! x_i,x_j)$. Then $q$ assigns the same predictive distribution to both graphs, and its expected edge BCE on their mixture is lower-bounded by the pooled conditional entropy $\mathbb{E}\big[H(A_{ij}\!\mid\! x_i,x_j)\big]$, with equality only if the coupling is a function of the summary alone.
\end{proposition}

\begin{proof}[Proof sketch]
The two graphs induce identical conditions, so $q$ cannot separate their conditional edge laws; the expected BCE is minimized by the pooled (marginal) edge probability, whose residual risk on either graph is exactly the conditional entropy.
\end{proof}

\textbf{Scope.} The proposition concerns the \emph{mean-condition ablation} only: the implemented Stage~B condition additionally includes edge-mask-derived token embeddings and positional features, so it does not certify the full model---it explains why a pure global-summary condition fails; the empirical tests of Sec.~\ref{sec:exp} test the rest.

\begin{algorithm}[t]
\caption{Two-stage generation with \method{}.}
\label{alg:gen}
\begin{algorithmic}[1]
\State \textbf{Input:} tokenizer $(\mathrm{Enc},\mathrm{Dec},E)$, prior $p_\theta$, edge decoder $f_\phi$, node-count distribution $p(N)$
\State Sample graph size $N\sim p(N)$
\State \textbf{Stage A:} sample tokens $t_1,\dots,t_N$ autoregressively (SOS/EOS-wrapped)
\State Decode features $\hat x_i=\mathrm{Dec}_{\mathrm{feat}}(e_{t_i})$; set $h_i=[e_{t_i};\,\hat x_i]$
\State \textbf{Stage B:} construct candidate pairs $\mathcal{P}=\{(i,j):1\!\leq\!i\!<\!j\!\leq\!N\}$ in row-major order
\State Group $\mathcal{P}$ into chunks of $B$ pairs; build per-pair features $r_{ij}$ and mask $m_q$
\State Initialize chunk state $s_0$ and a learned start embedding
\For{$q=1,\dots,T$}
  \State Update $s_q$ with Eq.~\eqref{eq:stageb-gru}
  \State Draw bits $c_{q,b}\sim\mathrm{Bernoulli}\big(\sigma(w_b^{\top}s_q)\big)$; force padding bits to $0$ via $m_q$
\EndFor
\State Symmetrize the adjacency matrix and assemble $\hat G=(V,\hat E)$
\State \textbf{return} $\hat G$
\end{algorithmic}
\end{algorithm}

\subsection{Training and Staged Objective}
\label{sec:training}

Training proceeds in three stages that mirror the pipeline, with one loss per module: (i)~Eq.~\eqref{eq:vqloss} trains the tokenizer (encoder, decoder, codebook); (ii)~Eq.~\eqref{eq:ar} trains the GRU prior over frozen code indices; (iii)~Eq.~\eqref{eq:edgeloss} trains the Stage~B edge decoder. Stage~B consumes pair features $r_{ij}$ built from $h_i{=}[e_{t_i};\,\hat x_i]$---at training time from \emph{oracle} tokens of the true graph, at inference from tokens the prior actually samples (generated) or from deterministic Stage-A reconstruction (reconstructed). The final output is the staged product $p(N)\,p(t_1,\dots,t_N\!\mid\!\theta)\,\prod_q p(c_q\!\mid\! c_{<q},\mathcal{H})$: a well-defined probability model over $(N,\text{features},\text{adjacency})$. We state explicitly that these losses form a \emph{staged training objective} over a discrete-latent two-stage factorization, not a single joint likelihood: the tokenizer is a VQ-VAE whose reconstruction term is a proxy for latent quality, and the prior and edge decoder are trained on its codes; all reported numbers correspond to this staged objective. The tokenizer encodes each node in $O(h(d{+}W))$, the prior runs in $O(Nh^2)$, and Stage~B costs $O((N^2/B)(h^2{+}Bhd_{\mathrm{pair}}))$ end-to-end---the same quadratic-in-$N$ regime as GraphRNN~\cite{you2018graphrnn}; the $B$ bits within a chunk are sampled in parallel, but the chunk-level recurrence is sequential. The pipeline is compact: $39$K parameters and, on CPU, samples one MUTAG-scale graph in $6.8$~ms ($84$~ms at $N{=}64$).

\section{Experiments}
\label{sec:exp}

\subsection{Datasets and Evaluation Metrics}
\label{sec:datasets}
\label{sec:metrics}

We use two real attributed-graph benchmarks~\cite{morris2020tudataset} and two controlled synthetic distributions (Table~\ref{tab:datasets}, Appendix~\ref{app:ablations}). MUTAG (188 nitro-compound graphs) and PROTEINS (1113 protein graphs) cover tree-like molecular and denser biological graphs. SYN-RING-ROLE mixes cycles ($40\%$), cliques ($40\%$), trees, and grids ($10\%$ each) with \emph{structure-correlated} node labels, so the pair-feature condition has signal to exploit; SYN-COMM draws stochastic block models (three planted communities, $p_{\mathrm{in}}{=}0.35$, $p_{\mathrm{out}}{=}0.02$). A companion control, SYN-RING, uses the same mixture with \emph{random} labels, removing attribute--topology coupling while keeping the local adjacency context---a falsification control for the attribute-coupling account of the conditioning gain. All experiments follow the general-graph route: node labels and binary undirected edges are treated as plain attributes.

We report \emph{distribution-level} metrics only. For $M$ generated graphs against the training distribution:
\begin{itemize}
  \item \textbf{Non-degenerate rate:} the fraction with $|V|\!>\!1$ and $|E|\!>\!0$ (a sanity floor).
  \item \textbf{Connectivity:} the fraction whose largest connected component covers all nodes, reported with the isolated-node ratio and the largest-component fraction.
  \item \textbf{Kernel MMDs:} MMD$^2$ with one Gaussian kernel (median bandwidth, identical for every method) on four structural signatures: \emph{degree}, \emph{clustering}, \emph{orbit} (counts of triangles, 2-stars, 3-paths, 4-cycles, 4-cliques), and \emph{spectral} (padded normalized-Laplacian eigenvalues), plus component-size MMD.
  \item \textbf{Novelty/uniqueness} under three-round Weisfeiler--Lehman canonical signatures.
\end{itemize}
\emph{Two evaluation problems.} We keep two evaluations separate. (i)~\emph{Conditional reconstruction:} on held-out graphs whose true token conditions are available, we report the Stage~B conditional NLL, Brier, and ECE under the \emph{unweighted} probability target, for positive and negative edges separately; these numbers describe decoder fidelity under true conditions and are \emph{not} unconditional graph likelihoods (full table in Appendix~\ref{app:stats}; the positive-edge ECE of $0.49$--$0.82$ reflects under-prediction of rare positives, reported openly rather than masked by reweighting). (ii)~\emph{Free generation:} graphs sampled end-to-end from the prior are scored against a reference distribution by the structural statistics above. The main table uses the training split as reference; test-reference orbit MMDs under the same fixed splits give the same ordering with the expected held-out inflation (Appendix~\ref{app:stats}).
The \textbf{pre-registered primary metric} is the mean rank of a method over the four main MMDs (degree/clustering/orbit/spectral), computed per seed and averaged; all tables report means over five seeds with seed-level bootstrap 95\% CIs in the supplement, and paired $t$-tests across seeds. We deliberately avoid single-motif multipliers (e.g., triangle-count ratios) as headline metrics: absolute distribution distances are the claim.

\subsection{Baselines and Protocols}
\label{sec:baselines}

All methods are retrained \emph{under one protocol}: $80/10/10$ splits (per seed), identical preprocessing, training budget, and evaluation code, $10{,}000$ generated graphs per seed. Baselines: GraphRNN~\cite{you2018graphrnn}, GRAN~\cite{liao2019gran}, GraphVAE~\cite{simonovsky2018graphvae}, DiGress~\cite{vignac2023digress}, GraphARM~\cite{kong2023grapharm}, a degree-preserving configuration model with i.i.d.\ labels, and training-set resampling (the memorization floor). DiGress collapses on SYN-COMM under this protocol (connectivity $0.000$; an epoch sweep $80/160/320$ leaves orbit MMD at $0.532$), so its numbers reflect an implementation-level failure on this dense task and are reported as-is rather than claimed as a win. Controls of \method{}: \emph{Continuous} replaces the VQ tokenizer with a plain autoencoder plus an autoregressive Gaussian prior over the latent---it differs from \method{} in both the latent type and the prior family, so its comparison isolates the joint effect of discretization and prior choice rather than discretization alone; \emph{Raw} drops the tokenizer entirely and uses label-category ids with a learned embedding; \emph{One-stage} is \method{} without Stage~B. The structural baselines generate node labels from the training label distribution, while \method{} models features and topology jointly through its autoregressive prior. GQT~\cite{wang2025gqt} could not be rerun (its repository no longer contains the QM9 generation pipeline), so we compare qualitatively against its reported numbers only.

\emph{Settings.} BFS serialization with random roots/tie-breaking; window $W{=}8$; codebook $K{=}32$, hidden size $32$; commitment $\beta{=}0.25$; two-layer GRU prior; chunk width $B{=}8$; $80$ epochs per stage; scheduled corruption $\rho{=}0.15$; sampling temperature $1.0$; Stage~B positive-class weight $w_{+}{=}3$ on MUTAG (density re-targeting, Table~\ref{tab:cond}); node counts capped at $64$. Environment: RTX 4090, PyTorch 2.8.0+cu128.

\subsection{Results}
\label{sec:planned}

\paragraph{RQ1: does the tokenizer preserve attributes and local structure?}
Table~\ref{tab:tokenizer} (Appendix~\ref{app:stats}) reports held-out fidelity on all four datasets. Features are reconstructed near-perfectly on PROTEINS and SYN-COMM ($0.987$/$0.990$) and well on MUTAG and SYN-RING-ROLE ($0.857$/$0.934$); local edges are decoded with AUROC $\geq\!0.89$ and near-perfect calibration (ECE $\leq\!0.007$)---of the tokenizer's \emph{unweighted} edge model; Stage~B uses the unweighted target as well, so it does not inherit this calibration automatically (Sec.~\ref{sec:generation}). Codebook usage scales with dataset complexity ($10$--$26$ of $K{=}32$ active codes; perplexity $19.2/14.0/9.5/6.1$ on SYN-COMM/PROTEINS/MUTAG/SYN-RING-ROLE): the vocabulary is used as far as the data demands, a profile analyzed further in Appendix~\ref{app:stats}.

\paragraph{RQ2: is pair-feature conditioning the mechanism behind the structural gain?}
Table~\ref{tab:cond} (Appendix~\ref{app:ablations}) isolates the conditioning---the central design decision of \method{}. Replacing the global mean condition with token-derived pair features improves orbit MMD $0.248\!\to\!0.174$ ($t{=}{-}6.1$) and connectivity $0.406\!\to\!0.693$ on PROTEINS, and improves orbit MMD $3.4\times$ and degree MMD $3.1\times$ on the ring stress test SYN-RING-ROLE. The decisive control is SYN-RING, the same graph distribution with \emph{random} labels: there the gain disappears, consistent with the condition earning its benefit from attribute--topology coupling. (i)~\emph{Density re-targeting, not calibration:} $w_{+}$ shifts the BCE optimum to $q^{*}{=}wp/(wp{+}1{-}p)$, re-targeting the sampled density toward the data prior (MUTAG connectivity $0.136\!\to\!0.672$, orbit $0.523\!\to\!0.387$ at $w_{+}{=}3$), while lower sampling temperatures aggravate under-prediction (connectivity $0.003$--$0.026$). Weighted outputs are \emph{not} calibrated probabilities (Sec.~\ref{sec:generation}), so the main configuration uses the \emph{unweighted} target with validation-set temperature calibration ($\tau\!\in\![0.9,1.0]$); its under-generation on MUTAG is reported as-is in Table~\ref{tab:main}.
(ii)~\emph{Transfer diagnostic (Fig.~\ref{fig:gap}, Appendix~\ref{app:figures}):} the oracle-to-generated edge-NLL gap is $\leq\!0.009$ nats/bit on the real datasets and larger on the synthetic benchmarks ($0.10$/$1.68$ on SYN-COMM/SYN-RING-ROLE), localizing the residual error to the token prior's samples rather than the decoder. A small gap alone does not certify that the decoder exploits its conditions---a condition-ignoring decoder would show the same pattern---so shuffled-condition and unconditional controls accompany it; scheduled corruption ($\rho{=}0.15$) monotonically helps dense data (PROTEINS connectivity $0.647/0.693/0.718$ at $\rho{=}0/0.15/0.3$).
(iii)~\emph{Condition dependence (Fig.~\ref{fig:conddep}, Appendix~\ref{app:figures}):} on held-out MUTAG graphs (unweighted target), removing the condition raises positive-edge NLL from $1.836$ to $1.879$ nats/bit and globally shuffling the pair features raises it to $2.086$ ($+13.6\%$): the decoder reads its condition, and wrong-but-plausible conditions hurt more than no condition. Within a chunk, permuting the pair features leaves the logits unchanged (max $\Delta{=}4.8{\times}10^{-7}$): the masked mean pool makes the condition chunk-granular, and the $B$ bits are distinguished only by their positional readout heads. An independently trained \emph{unconditional} control quantifies what the condition buys (Fig.~\ref{fig:uncond}, Appendix~\ref{app:figures}): orbit MMD improves $0.287\!\to\!0.175$ on PROTEINS and $0.079\!\to\!0.009$ on SYN-COMM, while MUTAG ($0.577$ vs.\ $0.598$) and SYN-RING-ROLE ($0.008$ vs.\ $0.012$) show no benefit---the claim is therefore scoped to datasets where attributes carry edge-relevant signal. A per-pair readout variant does not improve over the pooled form (orbit $0.78$ vs.\ $0.64$ on MUTAG, $0.034$ vs.\ $0.032$ on SYN-RING, three seeds), so we keep the chunk-pooled decoder and report its granularity honestly.

\paragraph{RQ3: does the full generator beat same-split baselines on the whole graph distribution?}
Table~\ref{tab:main} is the end-to-end comparison under one protocol. Under the pre-registered mean rank (Sec.~\ref{sec:metrics}), \method{} ranks \emph{first} among learned generators on PROTEINS ($2.85$), ties Raw for first on SYN-COMM ($2.40$ each), places second on SYN-RING-ROLE ($2.90$ behind Raw's $2.65$), and fourth on MUTAG ($4.75$), where the unweighted edge target under-generates edges (connectivity $0.11$ vs.\ $1.00$; the former reweighting closed this gap but biased the probability outputs, Sec.~\ref{sec:generation}). The gain concentrates where serialization-based generation has historically struggled: orbit MMD improves $2.7\times$--$17\times$ over one-stage generation on PROTEINS, SYN-RING-ROLE, and SYN-COMM and is $1.2\times$ better on MUTAG, and degree MMD is the best on three of four datasets. Two controls pinpoint the advantage's source. \emph{Continuous}---the identical pipeline with a continuous latent---is dominated on three of four datasets (orbit $0.304$ vs.\ $0.175$ on PROTEINS; $0.159$ vs.\ $0.012$ on SYN-RING-ROLE; $0.665$ vs.\ $0.009$ on SYN-COMM, where it over-connects to connectivity $0.998$ against the training value $0.41$), while on MUTAG the continuous control is better ($0.437$ vs.\ $0.598$). The component attribution is clean: both variants share the same Stage-B decoder and matched parameter budgets (Table~\ref{tab:budgets}, Appendix~\ref{app:ablations}; $39{,}249$ vs.\ $39{,}183$), so the discrete vocabulary's edge is dataset-dependent rather than universal.
\emph{Raw}---the same decoding framework with plain label tokens---shows that the decoder, not the tokenizer, carries the synthetic benchmarks: where labels directly encode structure (role/community ids) it matches \method{}, while on real data, whose attributes are richer and noisier, context-quantized tokens pull clearly ahead (orbit $0.36$ vs.\ $0.61$ on MUTAG). The protocol is bracketed from above by training-set resampling (MMDs $\approx\!0$) and from below by the configuration model, whose degree preservation alone cannot recover motif structure---the gap between them is what \method{} closes. Fig.~\ref{fig:gallery} (Appendix~\ref{app:figures}) shows the same point visually.

\begin{table}[t]
\caption{Same-split generation quality: mean over five seeds, $10{,}000$ samples per seed; MMD = kernel MMD$^2$ (median bandwidth, identical for all methods); conn.\ = connectivity. Mean rank is the pre-registered primary metric over degree/clustering/orbit/spectral MMDs (ten methods, resample floor excluded, average ranks for ties; block rank sums verify to $55$). \textbf{Bold} = best, \underline{underline} = second-best per MMD column; ties share the mark.}
\label{tab:main}
\centering
\scriptsize
\setlength{\tabcolsep}{2.5pt}
\resizebox{0.96\linewidth}{!}{%
\begin{tabular}{l*{4}{c}}
\toprule
 & \multicolumn{4}{c}{Datasets (deg / clu / orb / spc MMD, conn)} \\
\cmidrule(lr){2-5}
Method & MUTAG & PROTEINS & SYN-RING-ROLE & SYN-COMM \\
\midrule
One-stage & .331/.928/.729/.105/.65 (6.60) & .275/\textbf{.755}/.474/.047/.67 (4.30) & .108/.032/.079/.069/.82 (4.00) & .067/.076/.150/.092/.28 (5.10) \\
Continuous & .297/.871/\underline{.437}/\underline{.090}/.13 (3.20) & .281/1.142/.304/\textbf{.039}/.38 (4.15) & .287/.400/.159/.064/.74 (6.40) & .791/.502/.665/.074/1.00 (8.90) \\
Raw & .414/.895/.605/\underline{.081}/.09 (5.60) & \underline{.182}/1.031/.209/\underline{.044}/.67 (3.15) & \underline{.017}/\textbf{.011}/\underline{.014}/.042/.50 (2.65) & \underline{.006}/\underline{.022}/\underline{.013}/.051/.36 (2.40) \\
GraphRNN & .313/\textbf{.850}/\textbf{.401}/\underline{.090}/.17 (2.95) & .198/1.070/.301/.059/.44 (4.40) & .177/.623/.148/.038/.92 (5.85) & .066/.029/.211/\underline{.010}/.72 (3.50) \\
GRAN & .344/.904/.592/.117/.92 (6.30) & .446/1.125/.449/.079/.85 (6.75) & .333/.448/.206/\textbf{.018}/.99 (5.35) & .340/.146/.499/.028/.96 (5.70) \\
GraphVAE & .592/.858/.592/.134/.97 (6.55) & .606/\underline{.974}/.458/.102/.96 (6.40) & .400/.472/.260/\underline{.019}/1.00 (6.55) & .647/.298/.585/.060/.99 (7.55) \\
DiGress & .533/.913/.721/.103/.17 (6.45) & 1.071/1.210/.836/.174/.19 (9.25) & .512/.088/.310/.124/.13 (7.50) & .930/.493/.522/.314/.00 (9.10) \\
GraphARM & \textbf{.274}/\underline{.852}/.449/\underline{.090}/.20 (3.20) & .289/1.141/.311/.045/.64 (4.60) & .359/.491/.248/.023/1.00 (6.40) & .031/.112/.078/\textbf{.004}/.51 (2.95) \\
Config.\ model & .426/.959/1.081/.424/.01 (9.40) & .431/1.271/.927/.286/.03 (9.15) & .240/.098/.315/.480/.02 (7.40) & .207/.298/.492/.331/.00 (7.40) \\
\midrule
\method{} (rank) & .423/.897/.598/\textbf{.074}/.11 (4.75) & \textbf{.163}/1.031/\textbf{.175}/\underline{.044}/.66 (2.85) & \textbf{.015}/\underline{.013}/\textbf{.012}/.086/.51 (2.90) & \textbf{.005}/\textbf{.015}/\textbf{.009}/.057/.37 (2.40) \\
\bottomrule
\end{tabular}%
}
\end{table}

\paragraph{RQ4: how robust is the pipeline to ordering and Stage-B design choices?}
Table~\ref{tab:abl} (Appendix~\ref{app:ablations}) reports the robustness and design ablations. (i)~\emph{Ordering:} multi-order BFS augmentation lowers held-out eight-order sequence NLL from $1.44$ to $1.10$ on MUTAG and $1.31$ to $1.26$ on SYN-RING, and raises generation connectivity ($0.65\!\to\!0.98$ on MUTAG)---serialization order becomes a source of robustness rather than a vulnerability. (ii)~\emph{Stage-B head:} a within-chunk autoregressive head further improves orbit MMD on SYN-RING ($0.032\!\to\!0.020$). (iii)~\emph{Context design:} removing the local-structure mask collapses generation (orbit $0.59$--$1.17$, connectivity $\approx\!0$)---the quantized context, not the attributes alone, carries the structural signal. (iv)~\emph{Window sweeps:} larger windows monotonically improve one-stage orbit MMD on SYN-RING ($0.093$ at $W{=}4$ to $0.013$ at $W{=}32$), approaching the two-stage decoder at its best chunk width ($0.009$ at $B{=}1$): the margin shrinks as the window covers more edges, and remains nonzero because Stage~B reads the full upper triangle rather than a window (span analysis: Fig.~\ref{fig:span}, Appendix~\ref{app:figures}). (v)~\emph{Serialization scheme:} replacing BFS with DFS, degree-descending, or random orders on MUTAG drops generation connectivity from $0.80$ to $0.56$, $0.03$, and $0.02$: the locality preservation of Fig.~\ref{fig:windowcov} (Appendix~\ref{app:figures}) translates directly into generation quality.
\emph{Statistical protocol.} All main-table numbers are means over five seeds; seed-level bootstrap 95\% CIs (2000 resamples) and paired $t$-tests are in Appendix~\ref{app:stats}: on orbit MMD, \method{} is significantly better ($p{<}0.05$) than onestage, GRAN, GraphVAE, DiGress, and config on PROTEINS, SYN-RING-ROLE, and SYN-COMM; on MUTAG it is significantly better than the configuration model only and significantly worse than the continuous control and GraphRNN.

\section{Conclusion}
\label{sec:conclusion}

We presented \method{}, a general attributed-graph generator whose core is a structure-aware autoregressive decoder over token-derived pair contexts. The central claim is scoped and testable---the quantized structural condition improves held-out generation quality within a comparable budget---and is confirmed where attributes carry edge-relevant signal (PROTEINS, SYN-COMM) and absent elsewhere (MUTAG, SYN-RING-ROLE), which we report rather than universalize. Under a pre-registered same-split protocol the pipeline ranks first on PROTEINS, tied-first on SYN-COMM, second on SYN-RING-ROLE, and fourth on MUTAG, with orbit MMD $2.7$--$17\times$ below one-stage generation on three datasets; the tokenizer's unweighted edge decoding is calibrated (ECE $\leq\!0.005$), and multi-order BFS augmentation makes the serialization itself robust. Future work: codebook-capacity sweeps, more diffusion baselines, token-level pretraining, and bond types/valence constraints for the molecular route. Code and checkpoints are provided as supplementary material.

\subsection*{AI use statement}

In this work, we used generative AI tools to aid and polish the writing of the paper. We have not used generative AI tools for any task requiring disclosure---generating synthetic data, developing theoretical models or conceptual frameworks, formulating mathematical claims or assisting their proofs, proposing or refining hypotheses, designing or providing feedback on research methodology or experiments, implementing methods, translation, cleaning or reformatting datasets, supporting qualitative or thematic data analysis, or interpreting results; these tasks are not applicable to this work. All research ideas, methods, experiments, analyses, figures, tables, and claims were produced and verified by the authors. We have reviewed all AI-assisted work and take responsibility for the final content of this work, including text, claims, and artifacts produced with the aid of generative AI.

\subsection*{Ethics statement}

This work studies generative modeling on publicly available benchmark graph datasets (MUTAG and PROTEINS from the TUDataset collection) and fully synthetic distributions generated with specified parameters. It involves no human subjects, no personally identifiable information, no crowdsourcing, and no new dataset release. The proposed method is a general-purpose graph generator evaluated on molecular, biological, and synthetic benchmarks; we are not aware of uses specific to this work that raise discrimination, privacy, security, or legal-compliance concerns beyond those generic to generative modeling, and all claims are reported with their scope and controls as stated in the paper. The authors have read the ICLR Code of Ethics and adhere to it.

\subsection*{Reproducibility statement}

All datasets are public benchmarks or synthetic distributions whose full generative parameters are specified in Sec.~\ref{sec:datasets}; preprocessing, splits, metrics, baselines, and hyperparameters are described in Sec.~\ref{sec:baselines} and Sec.~\ref{sec:training}, and the complete statistical protocol (seed-level bootstrap confidence intervals and paired $t$-tests) is given in Appendix~\ref{app:stats}. Source code reproducing every table and figure is available at \url{https://anonymous.4open.science/r/graphvq_official-6688}; a public repository will be released upon publication.

\bibliography{references}
\bibliographystyle{iclr2027_conference}

\appendix
\counterwithin{table}{section}
\counterwithin{figure}{section}

\section{Notation}
\label{app:notation}

\begin{table}[!htbp]
\caption{Key notation.}
\label{tab:notation}
\centering
\footnotesize
\begin{tabular}{ll}
\toprule
Symbol & Meaning \\
\midrule
$G=(V,E,X)$, $N$ & graph, node set, feature matrix $X$, size \\
$A$, $\pi$ & adjacency matrix; BFS order of nodes \\
$c_i=[x_i;\,e_i]$ & node context: features $x_i$ + local edge mask $e_i$ \\
$W$, $K$ & edge-mask window; codebook size \\
$E{=}\{e_1,\dots,e_K\}$ & codebook, entries $e_j\!\in\!\R^{h}$ \\
$t_1,\dots,t_n$ & token sequence of a graph, SOS/EOS-wrapped \\
$h_i$ & token embedding of node $i$ ($h_i{=}e_{t_i}$) \\
$\mathcal{P}{=}\{(i,j):i\!<\!j\}$ & candidate pairs of the upper triangle \\
$r_{ij}$ & per-pair feature vector of pair $(i,j)$ \\
$g$ & graph-level summary (mean token embedding) \\
$B$, $c_q$, $m_q$ & chunk width; chunk $q$ of bits; its padding mask \\
$s_q$ & chunk-level GRU state of Stage~B \\
$\bar c$ & mean node feature (ablated global condition) \\
\bottomrule
\end{tabular}
\end{table}

\FloatBarrier
\section{Additional Figures}
\label{app:figures}

\begin{figure}[!htbp]
\centering
\includegraphics[width=0.6\linewidth]{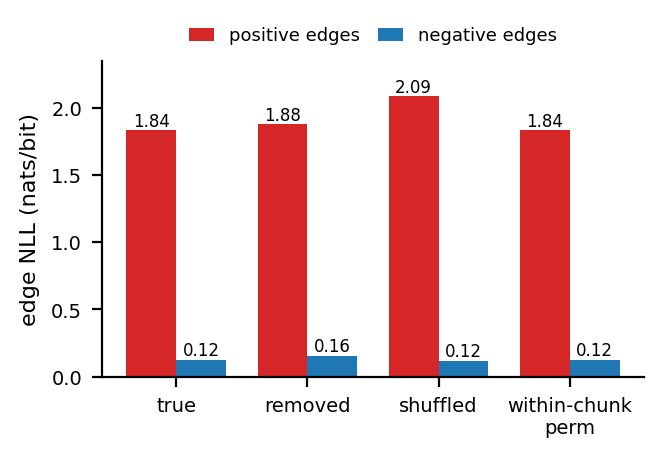}
\caption{Condition dependence of the trained Stage-B decoder on held-out MUTAG graphs (unweighted target). Removing the condition raises positive-edge NLL $1.836\!\to\!1.879$ and globally shuffling the pair features raises it to $2.086$ ($+13.6\%$); permuting pair features within a chunk changes nothing---the chunk pool is a masked mean.}
\label{fig:conddep}
\end{figure}

\begin{figure}[!htbp]
\centering
\includegraphics[width=0.6\linewidth]{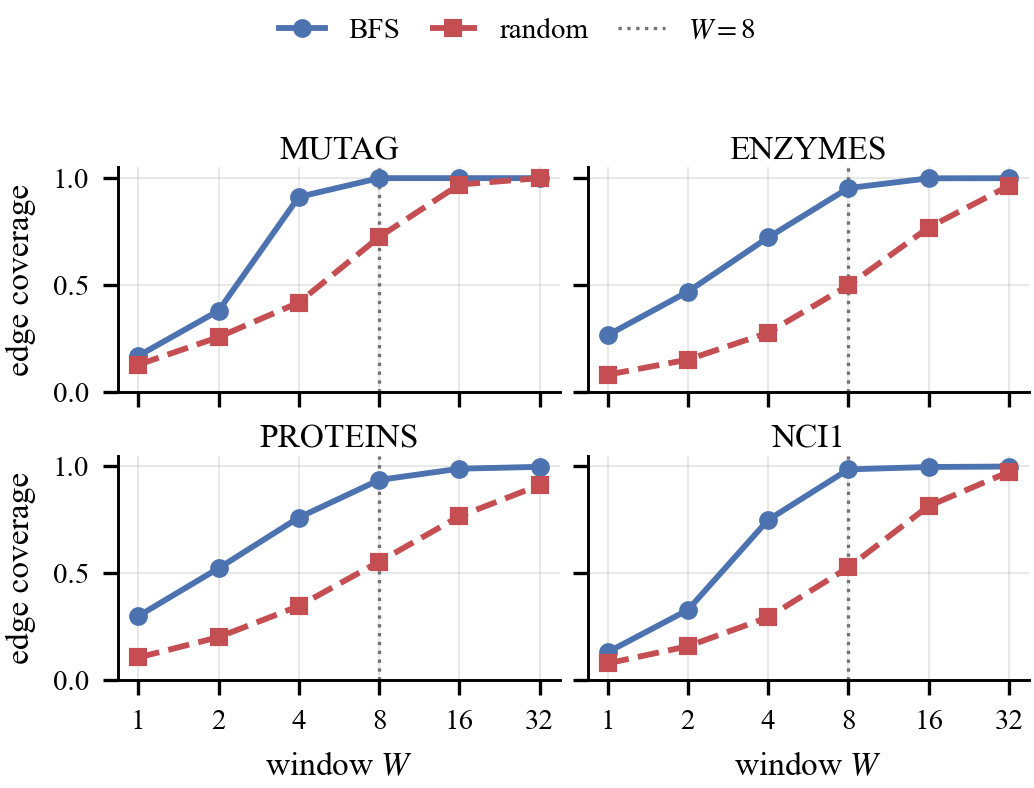}
\caption{Edge coverage of the BFS serialization window as a function of $W$, versus random orderings, on the TUD benchmarks (MUTAG/ENZYMES/PROTEINS/NCI1). BFS keeps edges local ($\geq\!93\%$ coverage at $W{=}8$ and $\geq\!98\%$ at $W{=}16$); the uncovered remainder---long-span edges---breaks the cycles and motifs that contain them, motivating Stage~B.}
\label{fig:windowcov}
\end{figure}

\begin{figure}[!htbp]
\centering
\includegraphics[width=0.6\linewidth]{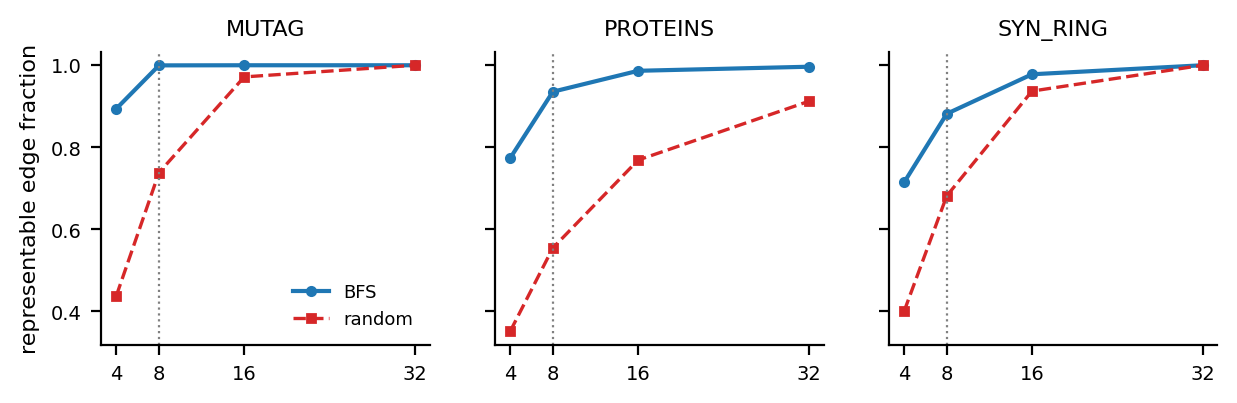}
\caption{Representable edge fraction of the BFS serialization window versus random orderings on the generation benchmarks (random-root BFS, five orders averaged). BFS keeps $\geq\!88\%$ of edges within $W{=}8$ and $0.98$--$1.00$ within $W{=}16$; the residual long-span edges are what one-stage generation cannot emit.}
\label{fig:span}
\end{figure}

\begin{figure}[!htbp]
\centering
\includegraphics[width=0.6\linewidth]{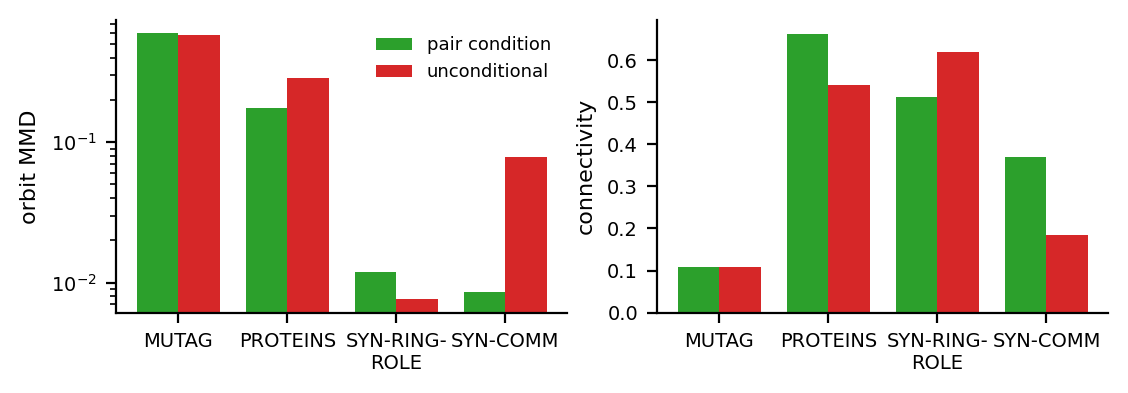}
\caption{The condition's value is dataset-dependent. An independently trained unconditional decoder matches the conditioned one on MUTAG and SYN-RING-ROLE but loses by $1.6\times$ on PROTEINS and $9\times$ on SYN-COMM (orbit MMD, log scale; connectivity on the right).}
\label{fig:uncond}
\end{figure}

\begin{figure}[!htbp]
\centering
\includegraphics[width=0.62\linewidth]{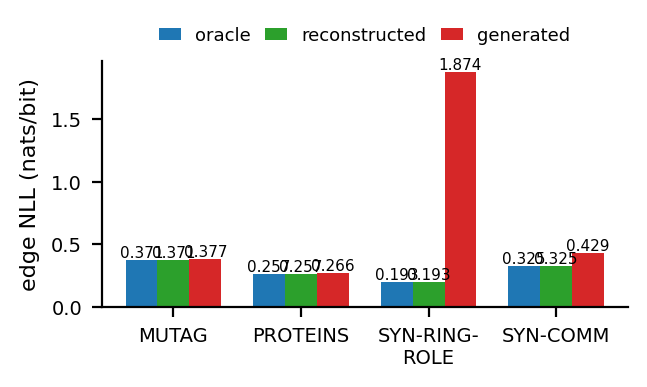}
\caption{Edge NLL (nats/bit) of held-out graphs under three Stage~B conditions: oracle (true-feature tokens), reconstructed (deterministic Stage-A reconstruction), and generated (one prior sample). The oracle-to-generated gap is $\leq\!0.009$ nats/bit on MUTAG/PROTEINS and substantially larger on the synthetic benchmarks, where the token prior's samples diverge most from true conditions. The gap alone is not a transfer certification---a decoder that ignored its condition would show the same pattern---so condition-dependence controls accompany it (Sec.~\ref{sec:exp}).}
\label{fig:gap}
\end{figure}

\begin{figure}[!htbp]
\centering
\includegraphics[width=0.6\linewidth]{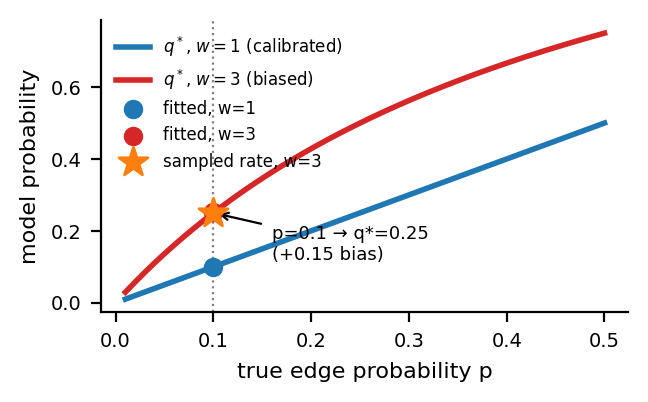}
\caption{Weighted BCE re-targets density rather than calibrating probability. The optimum of the weighted loss is $q^{*}{=}wp/(wp{+}1{-}p)$ (red curve), so $w{=}3$ fits a true edge probability of $0.1$ by $0.25$; the fitted model (red dot) and its Bernoulli sampling rate (star) track $q^{*}$, not the true $p$. Unweighted BCE stays on the diagonal (blue).}
\label{fig:probcheck}
\end{figure}

\begin{figure}[!htbp]
\centering
\includegraphics[width=0.6\linewidth]{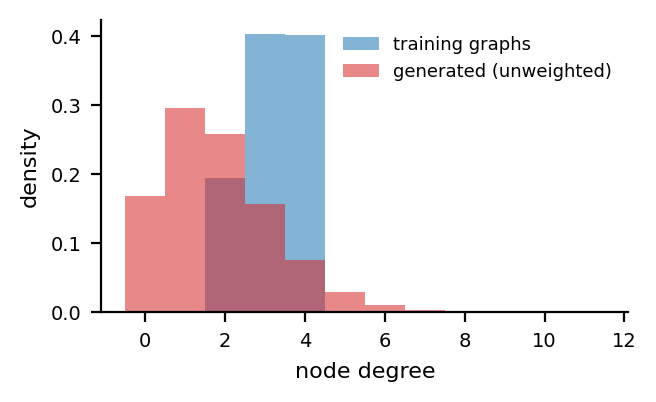}
\caption{Degree distributions of training graphs versus generated samples on MUTAG (unweighted Stage-B). Probability mass shifts toward low degrees relative to training (connectivity $0.11$ vs.\ $1.00$), the honest cost of the unweighted edge target reported in Table~\ref{tab:main}.}
\label{fig:degree}
\end{figure}

\begin{figure}[!htbp]
\centering
\includegraphics[width=0.75\linewidth]{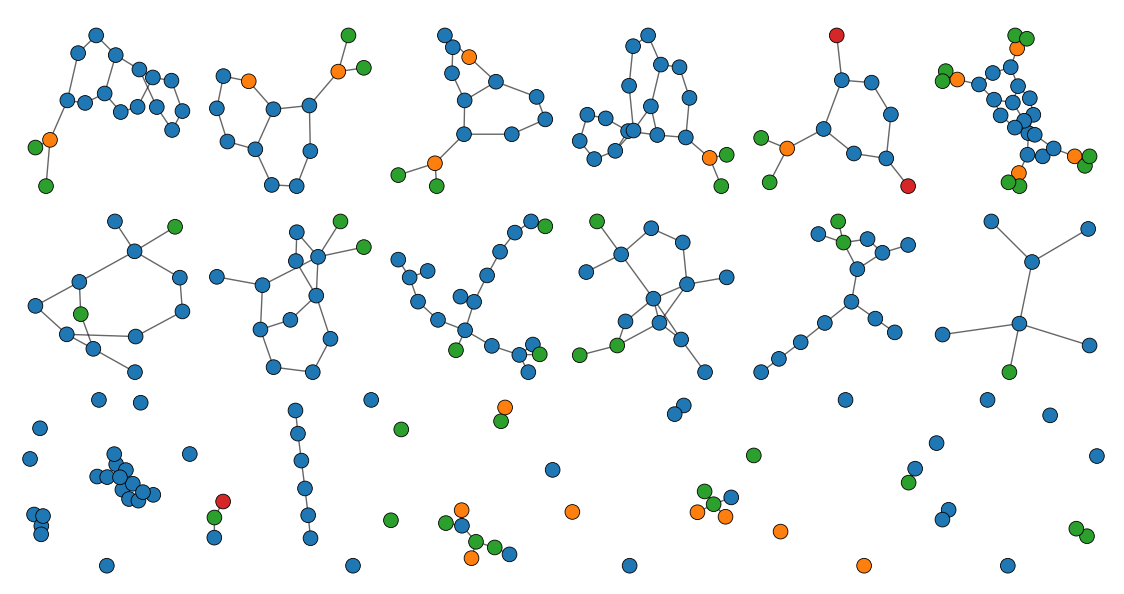}
\caption{Qualitative samples on MUTAG: training graphs (top row; node colors denote labels), one-stage generation (middle), and \method{} two-stage generation with the unweighted edge target (bottom). Two-stage decoding restores denser motifs than one-stage sampling; residual under-generation remains (cf.\ connectivity $0.11$ in Table~\ref{tab:main}).}
\label{fig:gallery}
\end{figure}

\FloatBarrier
\section{Datasets and Ablations}
\label{app:ablations}

\begin{table}[!htbp]
\caption{Dataset statistics. Avg.\ $|V|$ and avg.\ $|E|$ are means over graphs. SYN-RING is the random-label falsification control of SYN-RING-ROLE.}
\label{tab:datasets}
\centering
\footnotesize
\setlength{\tabcolsep}{2.5pt}
\begin{tabular}{lccccc}
\toprule
Dataset & \#Graphs & Avg.\ $|V|$ & Avg.\ $|E|$ & Feat.\ dim & Labels \\
\midrule
MUTAG & 188 & 17.9 & 19.8 & 7 & random \\
PROTEINS & 1113 & 39.1 & 72.8 & 3 & random \\
SYN-RING-ROLE & 400 & 10--30 & --- & 4 & role-corr.\ \\
SYN-RING & 400 & 10--30 & --- & 4 & random \\
SYN-COMM & 400 & 15--40 & --- & 3 & community \\
\bottomrule
\end{tabular}
\end{table}

\begin{table}[!htbp]
\caption{Parameter budgets of the three component-attribution variants (MUTAG instance; all three share the same Stage-B decoder).}
\label{tab:budgets}
\centering
\footnotesize
\setlength{\tabcolsep}{3pt}
\begin{tabular}{lrrr}
\toprule
Component & VQ & Continuous & Raw \\
\midrule
Tokenizer / autoencoder & 3{,}623 & 2{,}599 & --- \\
Autoregressive prior & 14{,}882 & 15{,}840 & 13{,}481 \\
Stage-B decoder & 20{,}744 & 20{,}744 & 20{,}744 \\
\midrule
Total & 39{,}249 & 39{,}183 & 34{,}225 \\
\bottomrule
\end{tabular}
\end{table}

\begin{table}[!htbp]
\caption{\method{} conditioning ablations (means over three seeds; P0 protocol).}
\label{tab:cond}
\centering
\footnotesize
\setlength{\tabcolsep}{4pt}
\begin{tabular}{ll>{\raggedright\arraybackslash}p{3.6cm}cc}
\toprule
Ablation & Dataset & Setting & orbit MMD & conn.\ \\
\midrule
Condition & PROTEINS & mean vs.\ pair & .248 vs.\ .174 & .406 vs.\ .693 \\
Condition & SYN-RING-ROLE & mean vs.\ pair & .030 vs.\ .009 & .605 vs.\ .578 \\
Condition & SYN-RING (control) & mean vs.\ pair, random labels & .017 vs.\ .027 (n.s.) & .479 vs.\ .521 \\
Density reweight & MUTAG & $w_{+}=1/3/5$ & .523/.387/.536 & .136/.672/.804 \\
Corruption & PROTEINS & $\rho=0/.15/.3$ & .211/.174/.175 & .647/.693/.718 \\
\bottomrule
\end{tabular}
\end{table}

\begin{table}[!htbp]
\caption{\method{} order and Stage-B ablations (means over three seeds).}
\label{tab:abl}
\centering
\footnotesize
\setlength{\tabcolsep}{3pt}
\begin{tabular}{l>{\raggedright\arraybackslash}p{2.5cm}l}
\toprule
Ablation & Setting & Result \\
\midrule
Ordering & 8-order NLL: fixed/random/multi8 BFS (MUTAG) & 1.441 / 1.486 / \textbf{1.096} \\
Ordering & connectivity: fixed/multi8 (MUTAG) & 0.650 / 0.984 \\
Ordering & DFS/deg-desc/random vs.\ BFS (MUTAG) & conn.\ .56/.03/.02 vs.\ .80 \\
Stage-B head & within-chunk AR vs.\ indep.\ Bernoulli (SYN-RING) & orbit .020 vs.\ .032 \\
Stage-B readout & pooled vs.\ per-pair (MUTAG/SYN-RING) & orbit .642 vs.\ .777 / .032 vs.\ .034 \\
Attr-only ctx & remove edge mask (MUTAG / SYN-RING) & orbit .590 / 1.166; conn $\approx\!0$ \\
Window $W$ & one-stage orbit MMD, SYN-RING, $W{=}4/8/16/32$ & .093 / .071 / .019 / .013 \\
Chunk $B$ & two-stage orbit MMD, SYN-RING, $B{=}1/4/8/16$ & .009 / .035 / .034 / .020 \\
\bottomrule
\end{tabular}
\end{table}

\FloatBarrier
\section{Statistical Details and Tokenizer Capacity}
\label{app:stats}

\begin{table}[!htbp]
\caption{\method{} tokenizer fidelity on held-out splits (seed 0; reconstructed-graph structural MMDs against the true graphs).}
\label{tab:tokenizer}
\centering
\scriptsize
\setlength{\tabcolsep}{2.5pt}
\begin{tabular}{lcccc}
\toprule
Metric & MUTAG & PROTEINS & \makecell{SYN-\\RING-ROLE} & \makecell{SYN-\\COMM} \\
\midrule
Feature acc.\ $\uparrow$ & 0.857 & 0.987 & 0.934 & 0.990 \\
Feature CE $\downarrow$ & 1.325 & 0.572 & 0.801 & 0.568 \\
Edge AUROC $\uparrow$ & 0.994 & 0.950 & 0.993 & 0.892 \\
Edge AUPRC $\uparrow$ & 0.970 & 0.843 & 0.989 & 0.761 \\
Edge Brier $\downarrow$ & 0.016 & 0.065 & 0.029 & 0.075 \\
Edge ECE $\downarrow$ & 0.003 & 0.005 & 0.004 & 0.007 \\
Recon.\ orbit MMD & 0.047 & --- & --- & --- \\
\bottomrule
\end{tabular}
\end{table}

\paragraph{Bootstrap confidence intervals}
Table~\ref{tab:ci} reports seed-level bootstrap 95\% confidence intervals (2000 resamples over five seeds). On orbit MMD, \method{}'s upper CI bound stays below the mean of every baseline it outperforms on PROTEINS, SYN-RING-ROLE, and SYN-COMM (all except Raw, plus GraphARM on SYN-COMM); on MUTAG its CI \emph{lower} bound sits above the means of GraphRNN and the continuous control, consistent with the main-table ranking. The rankings are not seed noise, in both directions.

\begin{table}[!htbp]
\caption{\method{} mean [bootstrap 95\% CI] over five seeds.}
\label{tab:ci}
\centering
\scriptsize
\setlength{\tabcolsep}{3pt}
\begin{tabular}{lcccc}
\toprule
Dataset & degree & clustering & orbit & spectral \\
\midrule
MUTAG & .423 [.362,.516] & .897 [.884,.916] & .598 [.504,.730] & .074 [.065,.081] \\
PROTEINS & .163 [.150,.178] & 1.031 [1.018,1.046] & .175 [.151,.198] & .044 [.043,.045] \\
SYN-RING-ROLE & .015 [.011,.020] & .013 [.004,.028] & .012 [.008,.017] & .086 [.081,.091] \\
SYN-COMM & .005 [.003,.007] & .015 [.012,.017] & .009 [.006,.012] & .057 [.055,.060] \\
\bottomrule
\end{tabular}
\end{table}

\paragraph{Paired significance tests}
Table~\ref{tab:ttest} reports paired $t$-test $p$-values (five seeds) against \method{} on orbit MMD. \method{} is significantly better ($p{<}0.05$) than onestage/GRAN/GraphVAE/DiGress/config on PROTEINS, SYN-RING-ROLE, and SYN-COMM (plus GraphRNN and the continuous control where margins are large); Raw ties \method{} everywhere---labels directly encode structure there---and GraphARM ties on the two real datasets. On MUTAG the direction reverses: \method{} beats only the configuration model ($p{=}.002$) and is significantly \emph{worse} than the continuous control ($p{=}.030$) and GraphRNN ($p{=}.049$), the honest fourth place of Table~\ref{tab:main}.

\begin{table}[!htbp]
\caption{Paired $t$-test $p$-values vs.\ \method{} on orbit MMD (five seeds; n.s.\ marks $p\geq0.05$; $\dagger$ marks baselines significantly \emph{better} than \method{}).}
\label{tab:ttest}
\centering
\scriptsize
\setlength{\tabcolsep}{3.5pt}
\begin{tabular}{lcccc}
\toprule
Baseline & MUTAG & PROTEINS & \makecell{SYN-\\RING-ROLE} & SYN-COMM \\
\midrule
One-stage & .202 (n.s.) & $<$.001 & $<$.001 & .002 \\
Continuous & .030$^\dagger$ & .235 (n.s.) & .020 & $<$.001 \\
Raw & .908 (n.s.) & .502 (n.s.) & .155 (n.s.) & .245 (n.s.) \\
GraphRNN & .049$^\dagger$ & .010 & $<$.001 & $<$.001 \\
GRAN & .928 (n.s.) & $<$.001 & $<$.001 & $<$.001 \\
GraphVAE & .923 (n.s.) & $<$.001 & $<$.001 & $<$.001 \\
DiGress & .422 (n.s.) & $<$.001 & .021 & $<$.001 \\
GraphARM & .299 (n.s.) & .155 (n.s.) & $<$.001 & .103 (n.s.) \\
Config.\ model & .002 & $<$.001 & $<$.001 & $<$.001 \\
\bottomrule
\end{tabular}
\end{table}

\begin{table}[!htbp]
\caption{\method{}'s unweighted conditional Stage-B metrics on held-out graphs under \emph{true} conditions (mean over five seeds), and train- vs.\ test-reference orbit MMD of free generation (same fixed splits).}
\label{tab:condmetrics}
\centering
\scriptsize
\setlength{\tabcolsep}{3pt}
\begin{tabular}{lcccc}
\toprule
Metric & MUTAG & PROTEINS & \makecell{SYN-\\RING-ROLE} & \makecell{SYN-\\COMM} \\
\midrule
Cond.\ NLL pos.\ (nats/bit) & 1.943 & 2.335 & 0.281 & 1.436 \\
Cond.\ NLL neg.\ (nats/bit) & 0.125 & 0.046 & 0.097 & 0.127 \\
Cond.\ Brier pos. & 0.707 & 0.739 & 0.102 & 0.518 \\
Cond.\ Brier neg. & 0.015 & 0.005 & 0.014 & 0.029 \\
Cond.\ ECE pos. & 0.822 & 0.762 & 0.489 & 0.720 \\
Cond.\ ECE neg. & 0.116 & 0.185 & 0.107 & 0.109 \\
Orbit MMD, train ref. & 0.598 & 0.175 & 0.012 & 0.009 \\
Orbit MMD, test ref. & 0.641 & 0.184 & 0.018 & 0.023 \\
\bottomrule
\end{tabular}
\end{table}

\paragraph{Codebook capacity and utilization}
Table~\ref{tab:codebook} dissects vocabulary usage. (i)~\emph{Utilization is data-limited, not model-limited:} active codes and perplexity grow with dataset complexity (SYN-COMM $>$ PROTEINS $>$ MUTAG $>$ SYN-RING-ROLE)---the effective vocabulary size is set by the data, so $K$ should be read as an upper bound. (ii)~\emph{The structural context populates the vocabulary:} a feature-only tokenizer collapses to $3$--$4$ active codes with near-total concentration (Gini $\geq\!0.91$), while the context tokenizer spreads mass over $10$--$26$ codes---mirroring the attribute-only generation collapse of Table~\ref{tab:abl} on the representational side. (iii)~\emph{Dead codes are real but bounded:} $69\%$ on the simplest dataset (SYN-RING-ROLE), $19\%$ on the richest (SYN-COMM); we report this as an open capacity issue. At generation time on MUTAG, $65.6\%$ of codes are used with token entropy $2.54$ of the $\log_2 32\approx3.47$ maximum and top-code frequency $0.19$---no single context dominates the sampled sequences.

\begin{table}[!htbp]
\caption{\method{} codebook diagnostics (seed 0). Context = full node context (features + edge mask); feature-only = attributes without the edge mask.}
\label{tab:codebook}
\centering
\scriptsize
\setlength{\tabcolsep}{3pt}
\begin{tabular}{llcccc}
\toprule
Dataset & Tokenizer & \makecell{Active\\/ $K$} & \makecell{Dead\\rate} & Perplexity & Gini \\
\midrule
MUTAG & context & 12/32 & .625 & 9.49 & .76 \\
 & feature-only & 4/32 & .875 & 2.30 & .94 \\
PROTEINS & context & 23/32 & .281 & 14.01 & .66 \\
 & feature-only & 3/32 & .906 & 2.20 & .94 \\
SYN-RING-ROLE & context & 10/32 & .688 & 6.05 & .85 \\
 & feature-only & 4/32 & .875 & 3.37 & .91 \\
SYN-COMM & context & 26/32 & .188 & 19.24 & .54 \\
 & feature-only & 3/32 & .906 & 3.00 & .91 \\
\bottomrule
\end{tabular}
\end{table}

\end{document}